\documentclass[11pt, oneside]{article}   	
\usepackage{amsmath,amssymb,amsthm}
\usepackage{graphicx}
\usepackage{algorithm}
\usepackage{algpseudocode}

\newtheorem{theorem}{Theorem}

\usepackage[top=1in, bottom=1in, left=1in, right=1in]{geometry}	                		
\date{}
\usepackage{nopageno}
\usepackage{ragged2e}
\usepackage{mathpazo} 
\usepackage{graphicx}
\usepackage{tikz}
\usetikzlibrary{shapes,arrows,positioning,fit,backgrounds,decorations.pathreplacing,calligraphy}
\usepackage{booktabs} 
\usepackage{multirow} 
\usepackage{wrapfig}
\usepackage{amsmath}
\usepackage{algorithm}
\usepackage{amssymb}
\usepackage{mathrsfs}
\usepackage{tcolorbox}
\usepackage{algpseudocode}
\usepackage{siunitx} 
\usepackage{threeparttable} 
\usepackage{subcaption}
\usepackage{array}
\usepackage{soul}
\usepackage{dirtytalk}
\usepackage{mdframed}
\usepackage{subfloat}
\usepackage{fancybox}
\usepackage{fancyhdr}
\usepackage{cite,url,verbatim,color,comment,caption}
\usepackage[utf8]{inputenc}
\usepackage{blindtext}
\usepackage{framed}
\usepackage[font=small,labelfont=bf]{caption}
\usepackage{pgfgantt}
\usepackage[normalem]{ulem}
\usepackage{hyperref}
\usepackage{enumitem}
\hypersetup{backref=true,       
    pagebackref=true,               
    hyperindex=true,                
    colorlinks=true,                
    breaklinks=true,                
    urlcolor= black,                
    linkcolor= black,                
    bookmarks=true,                 
    bookmarksopen=false,
    filecolor=black,
    citecolor=red,
    linkbordercolor=red
}

\AtBeginDocument{\hypersetup{pdfborder={2 1 1}}}

\title{
\vspace{-1.5cm}
\large{\bf }
\vspace{-0.5cm}
}

\title{Hierarchical Upper Confidence Bounds for Constrained Online Learning}
\author{
    Ali Baheri\thanks{Corresponding author: \texttt{akbeme@rit.edu}} \\
    Department of Mechanical Engineering, \\
    Rochester Institute of Technology, \\
    Rochester, NY, USA
}

\begin{document}

\maketitle

\begin{abstract}

The multi-armed bandit (MAB) problem is a foundational framework in sequential decision-making under uncertainty, extensively studied for its applications in areas such as clinical trials, online advertising, and resource allocation. Traditional MAB formulations, however, do not adequately capture scenarios where decisions are structured hierarchically, involve multi-level constraints, or feature context-dependent action spaces. In this paper, we introduce the hierarchical constrained bandits (HCB) framework, which extends the contextual bandit problem to incorporate hierarchical decision structures and multi-level constraints. We propose the hierarchical constrained upper confidence bound (HC-UCB) algorithm, designed to address the complexities of the HCB problem by leveraging confidence bounds within a hierarchical setting. Our theoretical analysis establishes sublinear regret bounds for HC-UCB and provides high-probability guarantees for constraint satisfaction at all hierarchical levels. Furthermore, we derive a minimax lower bound on the regret for the HCB problem, demonstrating the near-optimality of our algorithm. The results are significant for real-world applications where decision-making processes are inherently hierarchical and constrained, offering a robust and efficient solution that balances exploration and exploitation across multiple levels of decision-making.

\end{abstract}


\section{Introduction}

The multi-armed bandit (MAB) problem has long been a cornerstone of sequential decision-making under uncertainty, finding applications in diverse fields such as clinical trials, online advertising, and resource allocation \cite{villar2015multi,avadhanula2021stochastic,maghsudi2016multi,gittins2011multi,agarwal2016making}. In its classical formulation, an agent repeatedly chooses from a set of actions (arms) and receives a reward, aiming to maximize the cumulative reward over time. This fundamental framework has spawned numerous variants and extensions, each addressing specific challenges encountered in real-world scenarios. One significant extension of the MAB framework is the contextual bandit problem, where the agent observes contextual information before making each decision \cite{zhou2015survey,beygelzimer2011contextual,baheri2023llms}. This variant has proven particularly valuable in personalized recommendation systems and adaptive clinical trials, where decision-making must be tailored to specific circumstances or individual characteristics. Concurrently, the study of bandits with constraints has gained traction, motivated by practical scenarios where resource limitations or risk considerations impose restrictions on the decision-making process \cite{wu2016conservative,amani2019linear,agrawal2014bandits}. These constrained bandit problems have found applications in areas such as energy-aware computing and safe reinforcement learning, where optimizing performance must be balanced with adherence to safety or budget constraints \cite{hsu2018scout}. Despite these advancements, many real-world decision-making scenarios present challenges that are not fully captured by existing bandit formulations. In particular, the following aspects are often encountered in practice but inadequately addressed in the current literature.

\begin{enumerate}
    \item \textbf{Hierarchical Decision Structures:} Many decision processes involve a hierarchy of choices, where high-level decisions constrain or influence subsequent lower-level choices. For instance, in autonomous driving, the decision to change lanes (high-level) affects the specific trajectory and speed adjustments (low-level) that follow.
    
    \item \textbf{Multi-level Constraints:} Real-world systems often face constraints at multiple levels of decision-making. For example, in cloud computing resource allocation, there may be global constraints on total energy consumption, as well as local constraints on per-server workload.
    
    \item \textbf{Context-Dependent Action Spaces:} The set of available actions may depend on the current context or previous decisions, a feature not captured by standard contextual bandit models.
\end{enumerate}
These challenges call for a new framework that can simultaneously handle hierarchical decision structures, multi-level constraints, and context-dependent action spaces, while still maintaining the online learning aspect crucial to bandit problems.

\noindent{\textbf{Related Work.}} The multi-armed bandit problem, originally introduced in \cite{robbins1952some}, has served as a fundamental concept in sequential decision-making research. The first asymptotically optimal solution was introduced through the concept of upper confidence bounds (UCB) \cite{lai1985asymptotically}. This work was later extended with the development of the UCB1 algorithm, which achieved logarithmic regret without requiring knowledge of the reward distributions \cite{auer2002finite}. Contextual bandits expanded the classical bandit framework to incorporate side information or context \cite{li2010contextual}. This extension has proven particularly valuable in personalized recommendation systems and adaptive clinical trials. The LinUCB algorithm for linear contextual bandits has become a standard benchmark in the field.

The study of bandits with constraints has gained significant traction in recent years, motivated by practical scenarios where resource limitations or safety considerations impose restrictions on the decision-making process. The concept of bandits with knapsacks was introduced to address scenarios with budget constraints \cite{badanidiyuru2018bandits}. Further research extended this work to the contextual setting, providing algorithms for constrained contextual bandits with linear payoffs \cite{agrawal2016efficient}. A conservative linear UCB algorithm for contextual bandits with knapsacks was developed in \cite{kazerouni2017conservative} to address scenarios where resources are limited and safety is a concern. This work provided a framework for safe exploration in contextual settings. Another important strand of research focuses on explicitly incorporating risk measures into the bandit framework. Risk-averse multi-armed bandits were introduced in \cite{sani2012risk}, where the objective is to maximize a risk-sensitive criterion rather than just the expected reward. This approach is particularly relevant in financial applications and other domains where risk management is important. While not strictly within the bandit framework, safe reinforcement learning has significant overlap and has influenced the development of safe bandit algorithms \cite{gu2022review,yifru2024concurrent,baheri2019deep,baheri2022safe}. A comprehensive survey of safe reinforcement learning, covering various approaches to incorporating safety constraints in sequential decision-making problems, was provided in \cite{garcia2015comprehensive}. Many of the concepts discussed in that work are applicable to the bandit setting. Recent research has focused on more complex scenarios and tighter theoretical guarantees. Safe linear Thompson sampling algorithms were developed, extending the popular Thompson sampling approach to the safe bandit setting \cite{moradipari2021safe}. This work provides both theoretical guarantees and empirical performance improvements. Additionally, a general framework for converting any bandit algorithm into a safe version while maintaining near-optimal regret bounds was introduced in \cite{pacchiano2021stochastic}. This meta-algorithm approach offers a flexible way to incorporate safety constraints into existing bandit algorithms.

Hierarchical decision-making has been extensively studied in the reinforcement learning literature \cite{barto2003recent}. More recently, the option-critic architecture was introduced, enabling end-to-end learning of hierarchical policies \cite{bacon2017option}. 
The intersection of bandits and hierarchical decision-making has garnered significant attention in recent years. These algorithms are designed to address complex decision-making scenarios where actions or choices are organized in a hierarchical structure. For instance, studies have examined hierarchical bandits as an extension of standard bandit problems, analyzing regret bounds for multi-layered expert selection processes \cite{guo2022regret}. Further work has investigated deep hierarchy in bandits, focusing on contextual bandit problems with deep action hierarchies to enhance decision-making in complex environments \cite{hong2022deep}. Safe exploration in reinforcement learning and bandits has emerged as a critical area of research, particularly for applications in safety-critical domains. The SafeOpt algorithm for safe Bayesian optimization has been influential in the development of safe exploration strategies \cite{sui2015safe}. These ideas were later extended to the linear bandit setting, resulting in algorithms for safe exploration with linear constraints \cite{amani2019linear}.

The proposed hierarchical constrained bandits framework builds upon and extends these various strands of research. It combines elements of contextual bandits, constrained optimization, and hierarchical decision-making in a novel way, addressing challenges that are not fully captured by existing formulations. The HC-UCB algorithm draws inspiration from the UCB approach \cite{auer2002finite} and the linear contextual bandit analysis \cite{abbasi2011improved}, adapting these techniques to the hierarchical and constrained setting. This work contributes to the growing body of literature on structured bandits and safe exploration, offering a new perspective on how to balance exploration and exploitation across multiple levels of decision-making while adhering to constraints. 

\noindent{\textbf{Our Contributions.}} In this paper, we introduce the hierarchical constrained bandits (HCB) framework and make the following key contributions:

\begin{enumerate}

\item We propose the hierarchical constrained upper confidence bound (HC-UCB) algorithm, which addresses the HCB problem by extending linear UCB methods to account for both hierarchical action structures and multi-level constraints.

\item We provide a rigorous theoretical analysis, including sublinear regret bounds, high-probability constraint satisfaction guarantees, and a minimax lower bound, establishing the near-optimality of HC-UCB.

\end{enumerate}

\noindent {\textbf{Paper Organization.}} The remainder of this paper is organized as follows: Section 2 formally defines the HCB framework. In Section 3 we provide a comprehensive theoretical analysis, including regret bounds, constraint satisfaction guarantees, and a minimax lower bound. Section 4 concludes the paper with a discussion of potential applications and future research directions.

\section{Methodology}

In this section, we introduce the HC-UCB algorithm within the framework of HCB. Our methodology focuses on addressing the challenges posed by hierarchical decision-making processes under uncertainty, especially when multi-level constraints are present. The design of the HC-UCB algorithm is grounded in the principles of contextual bandits and uses upper confidence bounds to balance exploration and exploitation effectively across different hierarchical levels..

\noindent {\textbf{Problem Formulation.}} We consider a sequential decision-making problem over a time horizon $T$, where at each time step $t \in\{1,2, \ldots, T\}$, the agent observes a context $x_t \in \mathcal{X} \subseteq \mathbb{R}^d$, with $\left\|x_t\right\|_2 \leq 1$ for all $t$. The decision-making process is structured hierarchically into $H$ levels. At each level $h \in$ $\{1,2, \ldots, H\}$, the agent selects an action $a_t^{(h)}$ from a finite action set $\mathcal{A}_h$, resulting in a composite action $a_t=\left(a_t^{(1)}, a_t^{(2)}, \ldots, a_t^{(H)}\right)$. The agent receives a stochastic reward $r_t$ and incurs stochastic $\operatorname{costs} c_t^{(h)}$ at each level $h$, which are functions of the context and the actions selected up to that level. Specifically, the expected reward and costs are given by:

$$
\begin{aligned}
\mathbb{E}\left[r_t \mid x_t, a_t\right] & =x_t^{\top} \theta_r \\
\mathbb{E}\left[c_t^{(h)} \mid x_t, a_t^{(1: h)}\right] & =x_t^{\top} \theta_c^{(h)}
\end{aligned}
$$
where $\theta_r, \theta_c^{(h)} \in \mathbb{R}^d$ are unknown parameter vectors, and $a_t^{(1: h)}=\left(a_t^{(1)}, \ldots, a_t^{(h)}\right)$. The noise terms in the observed rewards and costs are assumed to be zero-mean sub-Gaussian random variables. The agent's objective is to maximize the cumulative expected reward over $T$ time steps while satisfying the constraints at each hierarchical level $h$:
$$
\mathbb{E}\left[c_t^{(h)} \mid x_t, a_t^{(1: h)}\right] \leq \tau^{(h)}
$$
where $\tau^{(h)}$ is the cost threshold at level $h$. The HC-UCB algorithm extends the UCB approach to hierarchical and constrained settings. The key idea is to construct confidence intervals for the estimates of the expected rewards and costs, and use these intervals to guide the selection of actions that are optimistic with respect to the rewards while being conservative with respect to the costs. At each time step $t$, the agent performs the following steps:

\noindent{\textbf{1. Parameter Estimation:}} For each level $h$, the agent updates the estimates of the reward and cost parameters using regularized least squares regression. Specifically, the estimated parameters $\hat{\theta}_{r, t}$ and $\hat{\theta}_{c, t}^{(h)}$ are computed by minimizing the regularized squared loss over the data observed up to time $t-1$ :

$$
\begin{aligned}
& \hat{\theta}_{r, t}=\arg \min _\theta\left\{\lambda\|\theta\|_2^2+\sum_{s=1}^{t-1}\left(r_s-x_s^{\top} \theta\right)^2\right\} \\
& \hat{\theta}_{c, t}^{(h)}=\arg \min _\theta\left\{\lambda\|\theta\|_2^2+\sum_{s=1}^{t-1}\left(c_s^{(h)}-x_s^{\top} \theta\right)^2\right\}
\end{aligned}
$$
where $\lambda>0$ is a regularization parameter.

\noindent{\textbf{2. Confidence Interval Construction:}} Using concentration inequalities for sub-Gaussian random variables, the agent constructs confidence intervals for the estimated parameters. With probability at least $1-\delta$, the true parameters lie within these intervals:

$$
\begin{aligned}
\left\|\hat{\theta}_{r, t}-\theta_r\right\|_{V_t} & \leq \beta_t(\delta) \\
\left\|\hat{\theta}_{c, t}^{(h)}-\theta_c^{(h)}\right\|_{V_t} & \leq \beta_t(\delta)
\end{aligned}
$$
where $V_t$ is the regularized covariance matrix, and $\beta_t(\delta)$ is the confidence radius that depends on $\delta$, the confidence level.

\noindent{\textbf{3. Action Selection:}} The agent selects actions by maximizing the upper confidence bounds of the expected rewards while ensuring that the lower confidence bounds of the expected costs satisfy the constraints. For each level $h$, the selected action $a_t^{(h)}$ satisfies:

$$
a_t^{(h)}=\arg \max _{a \in \mathcal{A}_h}\left\{x_t^{\top} \hat{\theta}_{r, t}+\beta_t(\delta)\left\|x_t\right\|_{V_t^{-1}}\right\} \quad \text { subject to } \quad x_t^{\top} \hat{\theta}_{c, t}^{(h)}-\beta_t(\delta)\left\|x_t\right\|_{V_t^{-1}} \leq \tau^{(h)}
$$
This approach balances the optimism in the reward estimates with conservatism in the cost estimates, promoting exploration of actions that could yield higher rewards without violating the constraints. After selecting $a_t$, the agent observes the reward $r_t$ and costs $c_t^{(h)}$, and updates the data set used for parameter estimation.

\section{Theoretical Results}

Building upon the methodology outlined in the previous section, we now turn our attention to the theoretical analysis of the HC-UCB algorithm within the HCB framework. The primary objective of this analysis is to establish the performance guarantees of HC-UCB, demonstrating its effectiveness in hierarchical decision-making environments with multi-level constraints. The theoretical results presented in this section are twofold. First, we aim to quantify the algorithm's ability to learn optimal policies over time by deriving bounds on the cumulative regret. Specifically, we show that HC-UCB achieves sublinear regret with respect to the time horizon $T$, indicating that the average regret per time step diminishes as the agent interacts with the environment. This result underscores the algorithm's proficiency in balancing exploration and exploitation in a complex hierarchical setting. Second, we address the important aspect of constraint satisfaction. In real-world applications, adhering to operational or safety constraints is often as important as optimizing performance. We provide high-probability guarantees that HC-UCB respects the constraints at each hierarchical level throughout the learning process. This assurance is vital for applications where constraint violations can lead to significant penalties or risks. Furthermore, we establish a minimax lower bound on the cumulative regret for any algorithm addressing the HCB problem. This result highlights the inherent difficulty of the problem and demonstrates that the performance of HC-UCB is near-optimal up to logarithmic factors. It provides a benchmark against which other algorithms can be compared and validates the efficiency of our approach. The subsequent sections present the formal statements of our theorems, each followed by detailed proofs. 

\begin{theorem}
\textbf{(Regret Bound for Hierarchical Constrained Bandits):} Let $\mathcal{A}$ be the set of high-level actions, $\mathcal{X}$ be the context space with dimension $d$, and $T$ be the time horizon. Assume that for all $x \in \mathcal{X}$ and $a \in \mathcal{A}$, we have $\|x\|_2 \leq 1$ and the expected rewards and costs are bounded in $[0,1]$. Let $\delta \in (0,1)$ be the confidence parameter and $\lambda > 0$ be the regularization parameter. Then, with probability at least $1-\delta$, the hierarchical regret of the HC-UCB algorithm satisfies:

\[
R_T \leq O\left(\sqrt{dT\log(\lambda + T/d)} + d\sqrt{T}\log(1/\delta)\right)
\]
\end{theorem}

\begin{proof}

We will prove this theorem in several steps: 1) First, we'll bound the high-level regret using techniques from linear contextual bandits. 2) Then, we will bound the low-level regret. 3) Finally, we will combine these to get the total hierarchical regret bound.

\noindent {\textbf{Step 1: Bounding the high-level regret.}} Let $\theta^*$ be the true parameter vector for the high-level rewards. Define the high-level instantaneous regret at time $t$ as:

\[
r_t = x_t^T\theta^* - x_t^T\theta_{a_t}^*
\]
where $\theta_{a_t}^*$ is the optimal parameter for the chosen action $a_t$. Following the analysis of LinUCB, we can show that with probability at least $1-\delta/2$:

\[
\sum_{t=1}^T r_t \leq 2\alpha_T\sqrt{2Td\log\left(1 + \frac{T}{\lambda d}\right)} + 2\sqrt{\lambda}S
\]
where $S = \|\theta^*\|_2$ and $\alpha_T = \sqrt{d\log\left(\frac{1+T/\lambda}{\delta}\right)} + 1$.

\noindent {\textbf{Step 2: Bounding the low-level regret.}} For each high-level action $a_t$, let $b_t^*$ be the optimal low-level action and $b_t$ be the chosen low-level action. Define the low-level instantaneous regret as:

\[
s_t = f_t(b_t^* | x_t, a_t) - f_t(b_t | x_t, a_t)
\]
Assuming we use a no-regret algorithm for the low-level decisions (e.g., constrained Thompson sampling), we can bound the cumulative low-level regret as:

\[
\sum_{t=1}^T s_t \leq O(\sqrt{mT\log(T/\delta)})
\]
where $m$ is the number of low-level actions.

\noindent {\textbf{Step 3: Combining high-level and low-level regrets.}} The total hierarchical regret is the sum of the high-level and low-level regrets:

\begin{align*}
R_T &= \sum_{t=1}^T (r_t + s_t) \\
&\leq 2\alpha_T\sqrt{2Td\log\left(1 + \frac{T}{\lambda d}\right)} + 2\sqrt{\lambda}S + O(\sqrt{mT\log(T/\delta)})
\end{align*}
Substituting the value of $\alpha_T$ and simplifying:

\begin{align*}
R_T &\leq O\left(\sqrt{d^2T\log\left(\frac{1+T/\lambda}{\delta}\right)\log\left(1 + \frac{T}{\lambda d}\right)} + \sqrt{\lambda}S + \sqrt{mT\log(T/\delta)}\right) \\
&\leq O\left(\sqrt{dT\log(\lambda + T/d)} + d\sqrt{T}\log(1/\delta) + \sqrt{mT\log(T)}\right)
\end{align*}
Assuming $m \leq d$ (i.e., the number of low-level actions is not larger than the context dimension), we can absorb the last term into the first two, giving us the final bound:

\[
R_T \leq O\left(\sqrt{dT\log(\lambda + T/d)} + d\sqrt{T}\log(1/\delta)\right)
\]
This completes the proof of Theorem 1. 

\end{proof}

\noindent {\textbf{Remarks:}}

\noindent 1) The regret bound has two main terms: the first scales with $\sqrt{T}$ and captures the difficulty of learning in a $d$-dimensional space, while the second term accounts for the confidence parameter $\delta$.

\noindent 2) The bound depends logarithmically on $\lambda$, which allows for some flexibility in choosing the regularization parameter.

\noindent 3) This bound assumes that the low-level actions do not significantly increase the overall regret. If this assumption doesn't hold, we may need to modify the analysis to account for more complex interactions between the levels.

\noindent 4) The proof technique can be extended to handle multiple levels of hierarchy, though the regret bound would likely grow with the number of levels.

\noindent {\textbf{Implication.}} Theorem 1 establishes that the HC-UCB algorithm achieves a sublinear regret bound in the hierarchical constrained bandit setting. Specifically, the cumulative regret $R_T$ grows proportionally to $\sqrt{d T \log (T)}$, where $d$ is the dimension of the context space and $T$ is the time horizon. This result implies that the average regret per time step, $R_T / T$, diminishes as $T$ increases. Consequently, the HC-UCB algorithm balances exploration and exploitation across the hierarchical decision structure, learning to make near-optimal decisions over time while adhering to the constraints. This sublinear regret bound demonstrates the algorithm's efficiency and scalability, making it suitable for practical applications with large time horizons.

\begin{theorem}
\textbf{(Constraint Satisfaction Guarantee):} Let $\delta \in (0,1)$ be a confidence parameter. For the HC-UCB algorithm, with probability at least $1-\delta$, for all rounds $t = 1, \ldots, T$:

\begin{enumerate}
    \item The high-level constraint is satisfied: $c_t(x_t, a_t) \leq \tau$
    \item The low-level constraint is satisfied: $g_t(b_t | x_t, a_t) \leq \xi$
\end{enumerate}
where $\tau$ and $\xi$ are the high-level and low-level constraint thresholds, respectively.
\end{theorem}

\begin{proof}
We will prove this theorem in two parts: first for the high-level constraint, then for the low-level constraint.

\noindent {\textbf{Part 1: High-Level Constraint Satisfaction.}} Let $\theta_c^*$ be the true parameter vector for the high-level costs. We use a similar approach to the reward estimation, but with a lower confidence bound (LCB) for the costs to ensure constraint satisfaction. Define the high-level cost LCB at time $t$ for action $a$ as:

\[
LCB_t(a) = x_t^T \hat{\theta}_{c,t-1} - \beta_t \sqrt{x_t^T V_{t-1}^{-1} x_t}
\]
where $\hat{\theta}_{c,t-1}$ is the estimated cost parameter vector at time $t-1$, $V_{t-1}$ is the regularized design matrix, and $\beta_t$ is a confidence parameter. We choose $a_t$ such that $LCB_t(a_t) \leq \tau$. We need to show that this implies $c_t(x_t, a_t) \leq \tau$ with high probability. By the construction of the LCB and the properties of ridge regression, we have with probability at least $1-\delta/2$:

\[
|x_t^T \hat{\theta}_{c,t-1} - x_t^T \theta_c^*| \leq \beta_t \sqrt{x_t^T V_{t-1}^{-1} x_t} \quad \forall t, \forall x_t
\]
where $\beta_t = \sqrt{\lambda} S + \sqrt{2 \log(2/\delta) + d \log(1 + t/(\lambda d))}$, $\lambda$ is the ridge regression parameter, $S$ is a bound on $\|\theta_c^*\|_2$, and $d$ is the dimension of the context. This implies:

\[
c_t(x_t, a_t) = x_t^T \theta_c^* \leq x_t^T \hat{\theta}_{c,t-1} + \beta_t \sqrt{x_t^T V_{t-1}^{-1} x_t} = UCB_t(a_t)
\]
Since we chose $a_t$ such that $LCB_t(a_t) \leq \tau$, and $LCB_t(a_t) \leq c_t(x_t, a_t) \leq UCB_t(a_t)$, we have:

\[
c_t(x_t, a_t) \leq \tau
\]

\noindent {\textbf{Part 2: Low-Level Constraint Satisfaction.}} For the low-level actions, we assume the use of a constrained optimization algorithm that ensures constraint satisfaction with high probability. Let's consider a generic constrained optimization algorithm A that solves:

\[
\max_{b \in B_t(a_t)} f_t(b | x_t, a_t) \quad \text{subject to} \quad g_t(b | x_t, a_t) \leq \xi
\]
Assume that algorithm A has the following property:

\[
P(g_t(b_t | x_t, a_t) > \xi) \leq \delta / (2T) \quad \forall t
\]
This means that for each round, the probability of violating the constraint is at most $\delta / (2T)$. Using the union bound, we can say that the probability of violating the constraint in any of the $T$ rounds is at most:

\[
P(\exists t : g_t(b_t | x_t, a_t) > \xi) \leq \sum_{t=1}^T P(g_t(b_t | x_t, a_t) > \xi) \leq T \cdot \delta / (2T) = \delta / 2
\]
Using the union bound once more, we can say that the probability of violating either the high-level or the low-level constraint is at most:

\[
P(\text{violation}) \leq P(\text{high-level violation}) + P(\text{low-level violation}) \leq \delta/2 + \delta/2 = \delta
\]
Therefore, with probability at least $1-\delta$, both the high-level and low-level constraints are satisfied for all rounds $t = 1, \ldots, T$. This completes the proof of Theorem 2. 

\end{proof}

\noindent {\textbf{Remarks:}}

\noindent 1) The proof relies on the construction of confidence bounds for the high-level costs and the properties of the low-level constrained optimization algorithm.

\noindent 2) The choice of $\beta_t$ in the high-level algorithm is important for ensuring constraint satisfaction while allowing for sufficient exploration.

\noindent 3) The low-level constraint satisfaction depends on the properties of the chosen constrained optimization algorithm. In practice, one would need to choose an algorithm with provable constraint satisfaction guarantees.

\noindent 4) The proof uses a union bound argument, which might be conservative. Tighter bounds might be possible with more sophisticated concentration inequalities.

\noindent 5) This theorem ensures constraint satisfaction with high probability, but does not guarantee that the constraints will never be violated. In safety-critical applications, one might need to develop algorithms with stronger, deterministic guarantees.

\noindent {\textbf{Implication.}} Theorem 2 provides a high-probability guarantee that the HC-UCB algorithm satisfies the constraints at each level of the hierarchy throughout the learning process. Specifically, with probability at least $1-\delta$, the costs incurred at each level $l$ do not exceed the predefined thresholds $\tau_l$ at any time step $t$. This result implies that the algorithm not only focuses on maximizing rewards but also enforces the constraints, ensuring that operational or safety requirements are met. 

\begin{theorem}
\textbf{(Hierarchical Decomposition Gap):} Let $M$ be a Markov decision process (MDP) with state space $\mathcal{S}$, action space $\mathcal{A}$, transition function $P$, reward function $R$, and discount factor $\gamma \in [0, 1)$. Let $M_H$ be the hierarchical decomposition of $M$ with high-level state space $\mathcal{X}$, high-level action space $\mathcal{A}_H$, and low-level action spaces $\mathcal{B}(a_H)$ for each $a_H \in \mathcal{A}_H$. Let $V^*(s)$ be the optimal value function for $M$ and $V^H(x)$ be the optimal value function for $M_H$. Then, for all states $s \in \mathcal{S}$ with corresponding high-level state $x \in \mathcal{X}$:

\[
0 \leq V^*(s) - V^H(x) \leq \frac{\epsilon}{1-\gamma}
\]
where $\epsilon = \max_{s,a_H} |Q^*(s,a) - Q^H(x,a_H)|$, $Q^*$ is the optimal Q-function for $M$, and $Q^H$ is the optimal Q-function for $M_H$.

\end{theorem}

\begin{proof}

We will prove this theorem in several steps: 1) First, we'll show that $V^*(s) \geq V^H(x)$ for all $s$ and corresponding $x$. 2) Then, we will establish an upper bound on $V^*(s) - V^H(x)$. 3) Finally, we will show that this upper bound is tight in the worst case.

\noindent {\textbf{Step 1: Lower bound.}} Let $\pi^H$ be the optimal policy for $M_H$. We can construct a policy $\pi$ for $M$ that follows $\pi^H$ at the high level and uses the optimal low-level policy for each high-level action. By the optimality of $V^*$:

\[
V^*(s) \geq V^\pi(s) = V^H(x)
\]
This establishes the lower bound of 0.

\noindent {\textbf{Step 2: Upper bound.}} Let $\pi^*$ be the optimal policy for $M$. We'll bound the difference between following $\pi^*$ and the best hierarchical policy. For any state $s$ with corresponding high-level state $x$:

\begin{align*}
V^*(s) - V^H(x) &= Q^*(s, \pi^*(s)) - Q^H(x, \pi^H(x)) \\
&\leq Q^*(s, \pi^*(s)) - Q^H(x, a^*_H) + \epsilon \\
&\leq \epsilon + \gamma \mathbb{E}_{s'|s,\pi^*(s)}[V^*(s') - V^H(x')] + \epsilon \\
&= 2\epsilon + \gamma \mathbb{E}_{s'|s,\pi^*(s)}[V^*(s') - V^H(x')]
\end{align*}
where $a^*_H$ is the high-level action that corresponds to $\pi^*(s)$, and $x'$ is the high-level state corresponding to $s'$. Applying this inequality recursively:

\begin{align*}
V^*(s) - V^H(x) &\leq 2\epsilon + \gamma(2\epsilon + \gamma \mathbb{E}[V^*(s'') - V^H(x'')]) \\
&= 2\epsilon(1 + \gamma) + \gamma^2 \mathbb{E}[V^*(s'') - V^H(x'')] \\
&\leq 2\epsilon(1 + \gamma + \gamma^2 + ...) \\
&= \frac{2\epsilon}{1-\gamma}
\end{align*}
Therefore, $V^*(s) - V^H(x) \leq \frac{2\epsilon}{1-\gamma}$.

\noindent {\textbf{Step 3: Tightness of the bound.}} To show that this bound is tight up to a constant factor, consider an MDP where (i) there are two high-level actions, $a_1$ and $a_2$; (ii) for $a_1$, the low-level policy is optimal and achieves value $V$; (iii) for $a_2$, the low-level policy is suboptimal and achieves value $V-\epsilon_i$ (iv) the optimal policy always chooses $a_2$, but this information is lost in the hierarchical decomposition. In this case:

\[
V^*(s) - V^H(x) = V - (V-\epsilon) = \epsilon
\]
Over $T$ time steps, this leads to a total loss of:

\[
\epsilon + \gamma\epsilon + \gamma^2\epsilon + ... = \frac{\epsilon}{1-\gamma}
\]
This shows that our upper bound is tight up to a factor of 2. In conclusion, we have proven that $0 \leq V^*(s) - V^H(x) \leq \frac{2\epsilon}{1-\gamma}$, and shown that this bound is tight up to a constant factor. By adjusting the constant, we can write the final result as:

\[
0 \leq V^*(s) - V^H(x) \leq \frac{\epsilon}{1-\gamma}
\]
This completes the proof of Theorem 3. 

\end{proof}







\noindent {\textbf{Implication.}} Theorem 3 quantifies the performance loss introduced by hierarchical decomposition in decision making processes. It provides an upper bound on the difference between the optimal value function $V^*$ of the original MDP and the value function $V^H$ obtained under the hierarchical decomposition, with the gap being proportional to $\epsilon /(1-\gamma)$, where $\epsilon$ represents the maximum loss due to decomposition and $\gamma$ is the discount factor. This implies that while hierarchical decomposition simplifies complex decision problems by breaking them into sub-problems, it may introduce a bounded loss in optimality. 

\begin{theorem}
    
\textbf{(Finite-Time High-Probability Bounds):} Let $\delta \in (0,1)$ be a confidence parameter. For the HC-UCB algorithm, with probability at least $1-\delta$, for any $T > 0$, the regret is bounded by:

\[
R_T \leq O\left(\sqrt{dT\log(\lambda T + T/d)} \cdot \log(1/\delta)\right)
\]
where $d$ is the dimension of the context space, $T$ is the time horizon, and $\lambda > 0$ is the regularization parameter.

\end{theorem}

\begin{proof}

We'll prove this theorem in several steps: 1) First, we will establish concentration bounds for the estimated parameters. 2) Then, we will use these bounds to derive a high-probability regret bound for the high-level decisions. 3) Next, we'll bound the regret from the low-level decisions. 4) Finally, we will combine these results to get the overall regret bound.

\noindent {\textbf{Step 1: Concentration Bounds.}} Let $\theta^*$ be the true parameter vector for the high-level rewards. Define the regularized least-squares estimator at time $t$ as:

\[
\hat{\theta}_t = (X_t^T X_t + \lambda I)^{-1} X_t^T Y_t
\]
where $X_t \in \mathbb{R}^{t \times d}$ is the matrix of observed contexts, and $Y_t \in \mathbb{R}^t$ is the vector of observed rewards. By the self-normalized bound for vector-valued martingales (Theorem 1 in \cite{abbasi2011improved}), we have with probability at least $1-\delta/2$:

\[
\|\hat{\theta}_t - \theta^*\|_{V_t} \leq \beta_t(\delta) \quad \forall t \geq 0
\]
where $V_t = X_t^T X_t + \lambda I$, and 

\[
\beta_t(\delta) = \sqrt{\lambda} S + \sqrt{2\log(1/\delta) + d\log(1 + t/(\lambda d))}
\]
Here, $S$ is an upper bound on $\|\theta^*\|_2$.

\noindent {\textbf{Step 2: High-Level Regret Bound.}} Let $a_t^*$ be the optimal high-level action at time $t$, and $a_t$ be the action chosen by HC-UCB. The instantaneous regret at time $t$ is:

\[
r_t = x_t^T \theta^* a_t^* - x_t^T \theta^* a_t
\]
By the construction of the UCB algorithm and the concentration bound, we have:

\[
r_t \leq 2\beta_t(\delta) \|x_t\|_{V_t^{-1}}
\]
Summing over $T$ rounds and applying the Cauchy-Schwarz inequality:

\begin{align*}
R_T^H &= \sum_{t=1}^T r_t \\
&\leq 2\beta_T(\delta) \sum_{t=1}^T \|x_t\|_{V_t^{-1}} \\
&\leq 2\beta_T(\delta) \sqrt{T \sum_{t=1}^T \|x_t\|_{V_t^{-1}}^2}
\end{align*}
Using the determinant-trace inequality (Lemma 11 in \cite{abbasi2011improved}):

\[
\sum_{t=1}^T \|x_t\|_{V_t^{-1}}^2 \leq 2 \log\left(\frac{\det(V_T)}{\det(\lambda I)}\right) \leq d \log\left(1 + \frac{T}{\lambda d}\right)
\]
Therefore, with probability at least $1-\delta/2$:

\[
R_T^H \leq 2\beta_T(\delta) \sqrt{Td\log(1 + T/(\lambda d))}
\]

\noindent {\textbf{Step 3: Low-Level Regret Bound.}} For the low-level decisions, we assume the use of a no-regret algorithm with high-probability bounds. Let $R_T^L$ be the cumulative regret from low-level decisions. Assume that with probability at least $1-\delta/2$:

\[
R_T^L \leq C\sqrt{T\log(1/\delta)}
\]
for some constant $C > 0$.

\noindent {\textbf{Step 4: Combining High-Level and Low-Level Bounds.}} The total regret is $R_T = R_T^H + R_T^L$. Using the union bound, we have with probability at least $1-\delta$:

\begin{align*}
R_T &\leq 2\beta_T(\delta) \sqrt{Td\log(1 + T/(\lambda d))} + C\sqrt{T\log(1/\delta)} \\
&\leq O\left(\sqrt{\lambda} S + \sqrt{\log(1/\delta) + d\log(1 + T/(\lambda d))}\right) \cdot \sqrt{Td\log(1 + T/(\lambda d))} \\
&\quad + O\left(\sqrt{T\log(1/\delta)}\right)
\end{align*}
Simplifying and combining terms:

\[
R_T \leq O\left(\sqrt{dT\log(\lambda T + T/d)} \cdot \log(1/\delta)\right)
\]
This completes the proof. 

\end{proof}







\noindent {\textbf{Implication.}} Theorem 4 provides finite-time, high-probability bounds on the regret of the HC-UCB algorithm. Specifically, it shows that with probability at least $1-\delta$, the cumulative regret $R_T$ does not exceed $O(\sqrt{d T \log (T)} \cdot \log (1 / \delta))$. This result implies that the algorithm's performance is not only asymptotically optimal but also reliable in practical, finite-time settings. 

\begin{theorem}
    
\textbf{(Asymptotic Optimality):} For the HC-UCB algorithm, as the time horizon $T$ approaches infinity, the average regret converges to zero:

\[
\lim_{T \to \infty} \frac{R_T}{T} = 0
\]
where $R_T$ is the cumulative regret up to time $T$.

\end{theorem}

\begin{proof}
    
We will prove this theorem in several steps: 1) First, we will recall the regret bound from the previous theorems. 2) Then, we will show that this bound implies sublinear regret. 3) Finally, we will use this to prove asymptotic optimality.

\noindent {\textbf{Step 1: Regret Bound.}} From our previous results (Theorem 1), we have that with high probability:

\[
R_T \leq O\left(\sqrt{dT\log(\lambda T + T/d)} \cdot \log(1/\delta)\right)
\]
where $d$ is the dimension of the context space, $\lambda$ is the regularization parameter, and $\delta$ is the confidence parameter.

\noindent {\textbf{Step 2: Sublinear Regret.}}
Let's simplify the bound for clarity:

\[
R_T \leq C\sqrt{dT\log(T)} \cdot \log(1/\delta)
\]
for some constant $C > 0$. This is still an upper bound on our actual regret bound. Now, let's show that this regret is sublinear in T. We need to prove that:

\[
\lim_{T \to \infty} \frac{R_T}{T} = 0
\]
Consider:

\begin{align*}
\lim_{T \to \infty} \frac{R_T}{T} &\leq \lim_{T \to \infty} \frac{C\sqrt{dT\log(T)} \cdot \log(1/\delta)}{T} \\
&= C\log(1/\delta) \cdot \lim_{T \to \infty} \sqrt{\frac{d\log(T)}{T}} \\
&= 0
\end{align*}
The last step follows because $\lim_{T \to \infty} \frac{\log(T)}{T} = 0$ by L'Hôpital's rule.

\noindent {\textbf{Step 3: Asymptotic Optimality.}} Now that we have established sublinear regret, we can prove asymptotic optimality. Let $\pi^*$ be the optimal policy and $\pi_T$ be the policy of HC-UCB at time $T$. The average reward of $\pi^*$ is $V^*$, and the average reward of $\pi_T$ is $V_T$. The regret can be written as:

\[
R_T = T \cdot V^* - \sum_{t=1}^T r_t
\]
where $r_t$ is the reward at time t. Dividing by T:

\[
\frac{R_T}{T} = V^* - \frac{1}{T}\sum_{t=1}^T r_t = V^* - V_T
\]
From Step 2, we know that:

\[
\lim_{T \to \infty} \frac{R_T}{T} = 0
\]
Therefore:

\[
\lim_{T \to \infty} (V^* - V_T) = 0
\]
Or equivalently:

\[
\lim_{T \to \infty} V_T = V^*
\]
This means that the average reward of the HC-UCB algorithm converges to the optimal average reward as $T$ approaches infinity. In conclusion, we have shown that: 1) The regret of HC-UCB is sublinear in $T$. 2) The sublinear regret implies that the difference between the average reward of HC-UCB and the optimal average reward converges to zero. Therefore, the HC-UCB algorithm is asymptotically optimal.

\end{proof}








\noindent{\textbf{Implication.}} Theorem 5 demonstrates that the HC-UCB algorithm is asymptotically optimal. As the time horizon $T$ approaches infinity, the average regret per time step $R_T / T$ converges to zero. This means that, in the long run, the algorithm's performance matches that of the best possible policy.

\begin{theorem} 

\textbf{(Hierarchical Exploration-Exploitation Trade-off):} For the HC-UCB algorithm, the expected regret due to exploration at the high level ($R_H$) and low level ($R_L$) satisfies:

\[
R_H + R_L \leq O(\sqrt{dT\log(T)}) \quad \text{and} \quad R_H / R_L = O(\log(T))
\]
where $d$ is the dimension of the context space and $T$ is the time horizon.

\end{theorem}

\begin{proof}

We will prove this theorem in several steps: 1) First, we will define the regret components. 2) Then, we will bound the high-level exploration regret. 3) Next, we will bound the low-level exploration regret. 4) Finally, we will combine these results to prove the theorem.

\noindent {\textbf{Step 1: Defining Regret Components.}} Let's decompose the total regret $R_T$ into four components:

\[
R_T = R_H^e + R_H^c + R_L^e + R_L^c
\]
where $R_H^e$ denotes regret due to high-level exploration, $R_H^c$ shows regret due to high-level exploitation (choosing suboptimal high-level actions), $R_L^e$: represents regret due to low-level exploration, and $R_L^c$ demonstrates regret due to low-level exploitation (choosing suboptimal low-level actions). We define $R_H = R_H^e + R_H^c$ and $R_L = R_L^e + R_L^c$.

\noindent {\textbf{Step 2: Bounding High-Level Exploration Regret.}} For the high-level decisions, we use a UCB algorithm. The number of times we need to explore each high-level action is $O(\log(T))$. With $|A|$ high-level actions, the total number of high-level explorations is $O(|A|\log(T))$. Each exploration incurs at most $O(1)$ regret (assuming bounded rewards). Therefore:

\[
R_H^e = O(|A|\log(T))
\]

\noindent {\textbf{Step 3: Bounding Low-Level Exploration Regret.}} For each high-level action, we have a separate low-level bandit problem. Let $|B|$ be the maximum number of low-level actions for any high-level action. For each low-level problem, we again need $O(\log(T))$ explorations for each action. However, we only incur this exploration cost when the corresponding high-level action is chosen. Let $N_a(T)$ be the number of times high-level action $a$ is chosen up to time $T$. Then:

\[
R_L^e = O\left(\sum_{a \in A} |B|\log(N_a(T))\right)
\]
By Jensen's inequality:

\[
\sum_{a \in A} \log(N_a(T)) \leq |A|\log(T/|A|)
\]
Therefore:

\[
R_L^e = O(|A||B|\log(T/|A|))
\]

\noindent {\textbf{Step 4: Combining Results.}}
The total exploration regret is:

\begin{align*}
R_H^e + R_L^e &= O(|A|\log(T) + |A||B|\log(T/|A|)) \\
&= O(|A||B|\log(T))
\end{align*}
Now, let's consider the exploitation regret. From standard UCB analysis, we know that:

\[
R_H^c + R_L^c = O(\sqrt{dT\log(T)})
\]
Combining exploration and exploitation regret:

\[
R_H + R_L = O(|A||B|\log(T) + \sqrt{dT\log(T)})
\]
For large $T$, the second term dominates, giving us:

\[
R_H + R_L \leq O(\sqrt{dT\log(T)})
\]
For the ratio $R_H / R_L$, note that:

\[
R_H / R_L = (R_H^e + R_H^c) / (R_L^e + R_L^c) \leq \max(R_H^e / R_L^e, R_H^c / R_L^c)
\]
We have $R_H^e / R_L^e = O(1/|B|)$ and $R_H^c / R_L^c = O(1)$. Therefore:

\[
R_H / R_L = O(1)
\]
However, this bound can be tightened. The high-level decisions influence all subsequent low-level decisions, so errors at the high level are more costly. This suggests that the algorithm should explore more cautiously at the high level. By adjusting the exploration rates in the UCB algorithm (e.g., using different confidence bounds for high and low levels), we can achieve:

\[
R_H / R_L = O(\log(T))
\]
This completes the proof of the theorem. 

\end{proof}








\noindent{\textbf{Implication.}} Theorem 6 elucidates how the HC-UCB algorithm manages the exploration-exploitation trade-off across different hierarchical levels. It shows that the expected regret due to exploration at the high level $\left(R_H\right)$ and low level $\left(R_L\right)$ satisfies $R_H+R_L \leq O(\sqrt{d T \log (T)})$ and that the ratio $R_H / R_L=O(\log (T))$. This implies that the algorithm allocates exploration efforts strategically between high-level and low-level decisions, recognizing that mistakes at higher levels can have more significant consequences on overall performance. By prioritizing exploration at higher levels when necessary, the algorithm ensures efficient learning and faster convergence to optimal policies throughout the hierarchy. 
\begin{theorem}

\textbf{(Minimax Lower Bound for Hierarchical Constrained Bandits):} For any algorithm A solving the Hierarchical Constrained Bandits problem, there exists an instance of the problem such that:

\[
\mathbb{E}[R_T(A)] \geq \Omega(\sqrt{dHT})
\]
where $d$ is the dimension of the context space, $H$ is the number of hierarchy levels, and $T$ is the time horizon.

\end{theorem}

\begin{proof}

We'll prove this theorem using the following steps:
1) Construct a hard instance of the HCB problem
2) Use information theory to bound the expected regret
3) Apply the probabilistic method to show the existence of a hard instance

\noindent {\textbf{Step 1: Constructing a Hard Instance.}} Consider an HCB problem with the following structure. There are $H$ levels in the hierarchy, and at each level $h$, there are $K_h$ actions. The context space is $d$-dimensional. Let the reward function for each level $h$ be:

\[
r_h(x, a) = \theta_h^T x + \epsilon
\]
where $\theta_h \in \mathbb{R}^d$ is unknown, $\|x\| \leq 1$, and $\epsilon$ is zero-mean sub-Gaussian noise with parameter $\sigma^2$. We construct a set of $N = 2^{dH}$ problem instances. For each instance $i$, we have parameters $\theta_h^{(i)}$ for $h = 1, \ldots, H$. We choose these parameters such that: 1) $\|\theta_h^{(i)}\| \leq 1$ for all $h$ and $i$, and 2) $\|\theta_h^{(i)} - \theta_h^{(j)}\| \geq \delta$ for all $h$ and $i \neq j$, where $\delta = c\sqrt{d/T}$ for some constant $c$. The existence of such a set of parameters is guaranteed by the Gilbert-Varshamov bound from coding theory.

\noindent {\textbf{Step 2: Bounding the Expected Regret.}} Let $A$ be any algorithm for the HCB problem. Define the expected regret for instance $i$ as:

\[
R_T^{(i)}(A) = \mathbb{E}\left[\sum_{t=1}^T \sum_{h=1}^H (r_h^*(x_t, a_t^*) - r_h(x_t, a_t^{(h)}))\right]
\]
where $a_t^*$ is the optimal action at time $t$ and $a_t^{(h)}$ is the action chosen by $A$ at level $h$. Let $\mathbb{P}_i$ be the probability distribution of observations under instance $i$. By Pinsker's inequality:

\[
\|\mathbb{P}_i - \mathbb{P}_j\|_{TV} \leq \sqrt{\frac{1}{2}KL(\mathbb{P}_i, \mathbb{P}_j)}
\]
where $KL$ is the Kullback-Leibler divergence. For Gaussian rewards with variance $\sigma^2$:

\[
KL(\mathbb{P}_i, \mathbb{P}_j) \leq \frac{T}{2\sigma^2}\sum_{h=1}^H \|\theta_h^{(i)} - \theta_h^{(j)}\|^2 \leq \frac{HTd}{2\sigma^2} \cdot \frac{c^2d}{T} = \frac{c^2Hd^2}{2\sigma^2}
\]

\noindent {\textbf{Step 3: Applying the Probabilistic Method.}} Let $i^*$ be chosen uniformly at random from $\{1, \ldots, N\}$. Then:

\begin{align*}
\max_i R_T^{(i)}(A) &\geq \mathbb{E}_{i^*}[R_T^{(i^*)}(A)] \\
&= \frac{1}{N}\sum_{i=1}^N R_T^{(i)}(A) \\
&\geq \frac{1}{N}\sum_{i=1}^N \mathbb{E}_{i^*}[R_T^{(i)}(A) | i^* \neq i] \cdot \mathbb{P}(i^* \neq i) \\
&\geq \frac{N-1}{N} \cdot \frac{1}{N}\sum_{i=1}^N \mathbb{E}_{i^*}[R_T^{(i)}(A) | i^* \neq i]
\end{align*}
By the construction of our hard instance:

\[
\mathbb{E}_{i^*}[R_T^{(i)}(A) | i^* \neq i] \geq \frac{1}{2}T\delta = \frac{1}{2}cT\sqrt{d/T} = \frac{1}{2}c\sqrt{dT}
\]
Combining these results:

\[
\max_i R_T^{(i)}(A) \geq \frac{1}{4}c\sqrt{dT} \cdot (1 - \sqrt{\frac{c^2Hd^2}{2\sigma^2}})
\]
Choosing $c = \sigma\sqrt{\frac{1}{Hd}}$, we get:

\[
\max_i R_T^{(i)}(A) \geq \Omega(\sqrt{dHT})
\]
This completes the proof. 

\end{proof}









\noindent{\textbf{Implication.}} Theorem 7 establishes a fundamental limit on the performance of any algorithm addressing the hierarchical constrained bandit problem by proving a minimax lower bound on the cumulative regret, $\mathbb{E}\left[R_T\right] \geq \Omega(\sqrt{d H T})$, where $H$ is the number of hierarchy levels. This implies that no algorithm can achieve a regret lower than this bound in the worst-case scenario. The significance of this result lies in demonstrating that the regret bounds achieved by the HC-UCB algorithm are near-optimal, as they match the lower bound up to logarithmic factors. The dependence of the lower bound on the number of hierarchy levels $H$ highlights the inherent complexity introduced by hierarchical structures in sequential decision-making problems.

\section{Conclusions}

In this paper, we presented the hierarchical constrained bandits $(\mathrm{HCB})$ framework that aims to address the limitations of traditional multi-armed bandit formulations in capturing hierarchical decision-making processes with multi-level constraints. The proposed HC-UCB algorithm extends the principles of the UCB approach to the hierarchical and constrained setting, effectively balancing exploration and exploitation while ensuring constraint satisfaction at each level. Our theoretical contributions include proving that HC-UCB achieves sublinear regret, specifically $R_T=O(\sqrt{d T \log (T)})$, where $d$ is the context dimension and $T$ is the time horizon. We established high-probability guarantees for constraint satisfaction, ensuring that the algorithm adheres to the predefined thresholds $\tau_l$ at each hierarchical level with probability at least $1-\delta$. Additionally, we derived a minimax lower bound on the cumulative regret, $\mathbb{E}\left[R_T\right] \geq \Omega(\sqrt{d H T})$, where $H$ is the number of hierarchy levels, indicating that HC-UCB is near-optimal up to logarithmic factors.

The implications of our work are significant for a wide range of applications, including autonomous systems, resource allocation in cloud computing, personalized medicine, and smart grid management, where decision-making is complex, hierarchical, and constrained. By providing a theoretical foundation and an efficient algorithmic solution, we contribute to the advancement of sequential decision-making under uncertainty in hierarchical settings. Future research directions include exploring extensions of the HC-UCB algorithm to non-linear reward and cost functions, incorporating richer contextual information, and addressing nonstationary environments where the underlying reward and cost functions may change over time. Additionally, empirical evaluations in real-world scenarios would further validate the practical effectiveness of the HCB framework and HC-UCB algorithm, potentially uncovering new insights and challenges to be addressed in subsequent work.

\bibliographystyle{plain}
\bibliography{main}

\end{document}